\documentclass{article}

\usepackage{microtype}
\usepackage{graphicx}
\usepackage{subfigure}
\usepackage{booktabs} 

\usepackage{amssymb}
\usepackage{amsmath}
\usepackage{graphicx}
\usepackage{amsmath}
\usepackage{amsthm}
\usepackage{booktabs}
\usepackage{algorithm}
\usepackage{algorithmic}
\usepackage{enumitem}

\newtheorem{theorem}{Theorem}
\newtheorem{lemma}{Lemma}
\newtheorem{remark}{Remark}
\newtheorem{corollary}{Corollary}

\usepackage{hyperref}

\usepackage[accepted]{icml2020}

\icmltitlerunning{Residual Bootstrap Exploration for Bandit Algorithms}

\begin{document}

\twocolumn[
\icmltitle{Residual Bootstrap Exploration for Bandit Algorithms}



\icmlsetsymbol{equal}{*}

\begin{icmlauthorlist}
\icmlauthor{Chi-Hua Wang}{equal,Purdue}
\icmlauthor{Yang Yu}{equal,Purdue}
\icmlauthor{Botao Hao}{Princiton}
\icmlauthor{Guang Cheng}{Purdue}
\end{icmlauthorlist}

\icmlaffiliation{Purdue}{Department of Statistics, Purdue University, West Lafayette, USA}
\icmlaffiliation{Princiton}{
Department of Electrical Engineering,
Princeton University, New Jersey, USA}

\icmlcorrespondingauthor{Guang Cheng}{chengg@purdue.edu}

\icmlkeywords{Machine Learning, ICML}

\vskip 0.3in
]



\printAffiliationsAndNotice{\icmlEqualContribution} 

\begin{abstract}In this paper, we propose a novel perturbation-based exploration method in bandit algorithms with bounded or unbounded rewards, called residual bootstrap exploration (\texttt{ReBoot}). The \texttt{ReBoot} enforces exploration by injecting data-driven randomness through a residual-based perturbation mechanism. This novel mechanism captures the underlying distributional properties of fitting errors, and more importantly boosts exploration to escape from suboptimal solutions (for small sample sizes) by inflating variance level in an \textit{unconventional} way. In theory, with appropriate variance inflation level, \texttt{ReBoot} provably secures instance-dependent logarithmic regret in Gaussian multi-armed bandits. We evaluate the \texttt{ReBoot} in different synthetic multi-armed bandits problems and observe that the \texttt{ReBoot} performs better for unbounded rewards and more robustly than \texttt{Giro}
\cite{kveton2018garbage} and \texttt{PHE} \cite{kveton2019perturbed}, with comparable computational efficiency to the Thompson sampling method.
\end{abstract}

\section{Introduction}\label{sec:intro}
A sequential decision problem is characterized by an agent who interacts with an uncertain environment,
and maximizes cumulative rewards \citep{sutton2018reinforcement}.  To learn to make optimal decisions as soon as possible, the agent must balance between
exploiting the current best decision to accumulate instant rewards and executing an exploratory decision to optimize future rewards. In literature \citep{lattimore2018bandit}, a gold standard is that we should explore more on where we are not sufficiently confident. Thus, a principled way to understand and ultilize uncertain quantification stands in the core of sequential decision makings.

The most popular approaches with theoretical guarantees are based on the optimism principle. The class of upper confidence bound (UCB) type algorithm \citep{auer2002finite, abbasi2011improved} tends to maintain a confidence set such that the agent acts in an optimistic environment. The class of Thompson sampling (TS) type algorithm \citep{russo2018tutorial} maintains a posterior distribution over model parameters and then acts optimistically with respect to samples from it. However, those types of algorithms are known to be hard to generalize to structured problems \citep{kveton2018garbage}. Thus, we may ask the following question: 
\begin{center}
\emph{Can we design a better principle that can provably quantify uncertainties and is easy to generalize? }    
\end{center}

We follow the line of bootstrap explorations \citep{osband2015bootstrapped, elmachtoub2017practical, kveton2018garbage}, which is known to be easily generalized to structured problems. In this work, we carefully design a type of ``perturbed'' residual bootstrap as a data-dependent exploration in bandits, which may also be viewed as ``follow the bootstrap leader'' algorithm in a general sense. The main principle is that residual bootstrap in the statistics literature \citep{mammen1993bootstrap} can be adapted to capture the underlying distributional properties of fitting errors. It turns out that the resulting level of exploration leads to the optimal regret. In this case, we call the employed perturbation as ``regret-optimal perturbation'' scheme. The regret-optimal perturbation is obtained through appropriate uncertainty boosting based on residual sum of squares (RSS). 

\textbf{Our contributions:}
\vspace{-1.5 mm}
\begin{itemize}
    \item We propose a novel residual bootstrap exploration algorithm that maintains the generalization property and works for both bounded and unbounded rewards;  
    
    \item We prove an optimal instance-dependent regret for an instance of \texttt{ReBoot} for unbounded rewards (Gaussian bandit). This is a non-trivial extension beyond Bernoulli rewards \citep{kveton2018garbage}.  We utilize sharp lower bounds for the normal distribution function and carefully design the variance inflation level;
    
    \item We empirically compare \texttt{ReBoot} with several provable competitive methods and demonstrate our superior performance in a variety of unbounded reward distributions while preserving computational efficiency.
\end{itemize}
\textbf{Related Works.} \texttt{Giro} \cite{kveton2018garbage} directly perturbed the historical rewards by nonparametric bootstrapping and adding deterministic pseudo rewards. One limitation is that \texttt{Giro} is only suitable for bounded rewards. The range of pseudo rewards usually depends on the extreme value of rewards which could be $\pm\infty$ for an unbounded distribution. In this case, there is no principle guidance in choosing appropriate values for pseudo rewards. Technically, the analysis in \citet{kveton2018garbage} heavily relies on beta-binomial transformation,
and thus is hard to extend to unbounded reward, e.g. Gaussian bandit. Another related work \citet{elmachtoub2017practical} utilized bootstrap to randomize decision-tree estimator in contextual bandits but did not provide any optimal regret guarantee.

\texttt{PHE} \cite{kveton2019perturbed} randomized the history by directly injecting Bernoulli noise and provided regret guarantee limited to bounded reward case. \citet{lu2017ensemble} proposed ensemble sampling by injecting Gaussian noise to approximate posterior distribution in Thompson sampling. However, their regret guarantee has an irreducible term linearly depending on time horizon even when the posterior can be exactly calculated. In practice, it may also be hard to decide what kind of noises and what amount of noises should be injected since this is not a purely data-dependent approach.

Another line of works is to use bootstrap to construct sharper confidence intervals in UCB-type algorithm. \citet{NIPS2019_9382} proposed a nonparametric and data-dependent UCB algorithm based on the multiplier bootstrap and derived a second-order correction term to boost the agent away from sub-optimal solutions. However, their approach is computational expensive since at each round, they need to resample a large number of times for the history. By contrast, \texttt{Reboot} only needs to resample once at each round.

\textbf{Notations.} Throughout the paper, we denote $[n]$ as the set $\{1,2,\cdots, n\}$. For a set $E$, we denote its complement as $E^{c}$. We denote $N(\mu, \sigma^2)$ as the Gaussian distribution with mean parameter $\mu$ and variance parameter $\sigma^2$. We write $X \sim ( \mu, \sigma^2)$ for a random variable $X$ if its distribution has mean $\mu$ and variance $\sigma^2$. We write $a\lesssim b$ if $a\leq Cb$ for some constant $C$.


\section{Residual Bootstrap Exploration (\texttt{ReBoot})}\label{sec_ReBoot_alg}

In this section, we first briefly discuss in Section \ref{sec:vanilla_RB} why vanilla residual bootstrap exploration may not work for multi-armed bandit problems. This motivates the \texttt{ReBoot} algorithm as a remedy to be presented in Section \ref{sec:alg_reboot} and the full description of each step will be given in Section \ref{subsec:ReBootExploration_MAB}.
Discussion and interpretation on the tuning parameter $\sigma_{a}$ is given in Section \ref{subsec:prevent_overunder_exploration}.

\textbf{Problem setup.}
We present our approach in the stochastic multi-armed bandit (MAB) problems.  
There are $K$ arms, and each arm $k=1\in[K]$ has a reward distribution $P_k$ with an unknown mean parameter $\mu_k$. Without loss of generality, we assume arm 1 is the optimal arm, that is, $\mu_1 = \max_{k \in [K]}\mu_{k}$.
Specifically, the agent interacts with an bandit environment for $T$ rounds. In round $t \in [T]$, the agent pulls an arm $I_{t} \in [K]$ and observes a reward $r_{t}$. 
The objective is to minimize the expected cumulative regret, defined as, 
\begin{equation}
    R(T) = 
    T\mu_{1}
    -E[\sum_{t=1}^{T}r_{t}]
    =
    \sum_{k=2}^{K}\Delta_{k}E[\sum_{t=1}^{T}I\{I_{t} = k\}],
\end{equation}
where $\Delta_{k}=\mu_1 - \mu_k$ is the sub-optimality gap for arm $k$, and $I\{\cdot\}$ is an indicator function. Here, the second equality is from the regret decomposition Lemma (Lemma 4.5 in \citet{lattimore2018bandit}). 

\subsection{The failure of vanilla residual bootstrap} 
\label{sec:vanilla_RB}

Bootstrapping a size $s$ reward sample set of arm $k$, $\{Y_{k,i}\}_{i=1}^{s}$, via residual bootstrap method \cite{mammen1993bootstrap}
 consists of the following four steps:
\vspace{-2mm}
\begin{enumerate}[noitemsep]
\item Compute an average reward $\bar{Y}_{k,s}=s^{-1}\sum_{i=1}^{s}Y_{k,i}$. \item Compute the residuals $\{e_{k,i}\}_{i=1}^{s}$ for each reward sample with $e_{k,i} = Y_{k,i} - \bar{Y}_{k,s}$.
\item Generate bootstrap weights 
(random variables with zero mean and unit variance) $\{w_i\}_{i=1}^{s}$. 
\item 
Add the reward average $\bar{Y}_{k,s}$
with the average of perturbed residuals  
$s^{-1}\sum_{i=1}^{s} w_i e_{k,i}$ to get perturbed reward average $\bar{Y}_{k,s}^{*}$.
\end{enumerate}

\vspace{-1.5mm}
An exclusive feature of $\bar{Y}_{k,s}^{*}$ is that it preserves the empirical variation among current data set. 
That is, 
the conditional variance of perturbed reward average $\bar{Y}_{k, s}^{*}$ on given reward sample set $\{Y_{k,i}\}_{i=1}^{s}$ is the uncertainty quantified by reward average $\bar{Y}_{k,s}$.
To see this, notice that, from the above residual bootstrap procedure, the perturbed reward average $\bar{Y}_{k,s}^{*}$ admits the presentation
\begin{equation}\label{eq:perterbed_mean}
    \bar{Y}_{k,s}^{*} =
    \bar{Y}_{k,s}
    + \frac{1}{s}\sum_{i=1}^{s}w_i\cdot e_{k,i}.
\end{equation}
Since the bootstrap weights are required to have zero mean and unit variance, the distribution of perturbed average $\bar{Y}_{s}^{*}$ conditioning on the current data set $\{Y_{i}\}_{i=1}^{s}$ has mean and variance as
\begin{equation}\label{eq:pertb_ave}
    \bar{Y}_{k,s}^{*}|\{Y_{k,i}\}_{i=1}^{s}\sim (\bar{Y}_{k,s}, s^{-2}\text{RSS}_{k,s}),
\end{equation}
which means that its expectation equals sample average $\bar{Y}_{k,s}$ and its variance can be represented as $s^{-2}\text{RSS}_{k,s}$, 
where 
\begin{equation}\label{eq:RSS}
\text{RSS}_{k, s}= \sum_{i=1}^{s}e_{k,i}^2
\end{equation} is the residual sum of squares.

\begin{remark}
Note that the RSS in \eqref{eq:RSS} is a standard measure of \textit{goodness of fit} in statistics, that is, how well a statistical model fit the current data set. After a glimpse, it seems that the \textit{variance} of bootstrap-based mean estimator $\bar{Y}_{k,s}^{*}$ should drop hints on the \textit{right} amount of randomness for exploration. Such intuition guides a different type of exploration in MAB problem, as elaborated in the following paragraph.
\end{remark}

\textbf{Vanilla residual bootstrap exploration.}

As well recognized in the literature of bandit algorithm \cite{lattimore2018bandit}, policy using reward average as arm index (Follow-the-Leader algorithm) can incur linear regret in multi-armed bandit problem (Figure.\ref{fig:linear_regret}; blue line). Alternatively, policy using perturbed reward average via residual bootstrap as in \eqref{eq:pertb_ave} induces an 
\textit{data-driven} exploration, at the level of 
statistical uncertainty of current reward sample set ($s^{-2}\text{RSS}_{k,s}$).

Exploration at the level of current reward sample set's statistical uncertainty is a mixture of hopes and concerns. On one hand, we hope the data-driven exploration level ($s^{-2}\text{RSS}_{k,s}$) hints a right amount of randomness for escaping from suboptimal solutions; on the other hand, we concern that large-deviated average of reward samples or poor fitting of adopted statistical model haunts the performance of bandit algorithms.

\begin{figure}[t!]
\begin{center}
\centerline{\includegraphics[width=\columnwidth]{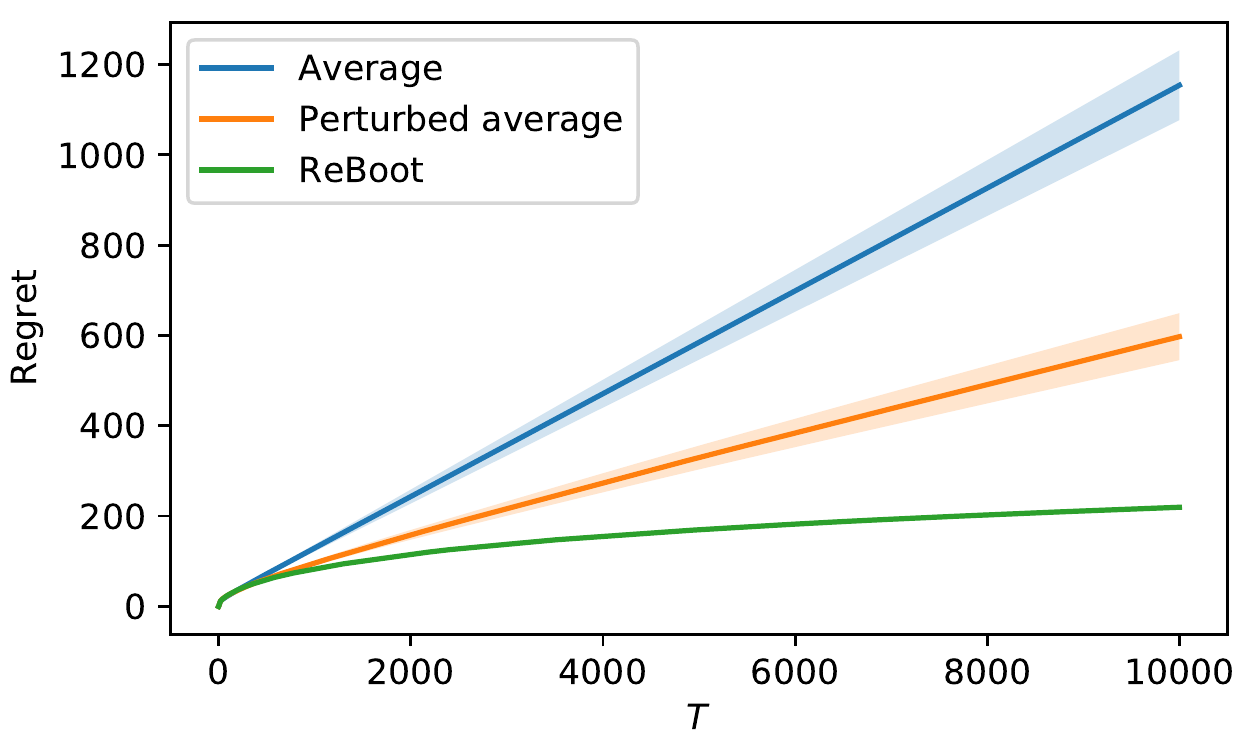}}
\vskip -0.2in
\caption{Gaussian 10-armed bandit using reward average (no exploration), perturbed reward average, and the proposed \texttt{ReBoot} algorithm. $\{\mu_k\}_{k=1}^{10}$ are randomly picked from $[-1,1]$. The error bars represent the standard error of the mean regret over $500$ runs. 
}
\label{fig:linear_regret}
\end{center}
\vskip -0.4in
\end{figure}

Unfortunately, we witness in empirical experiments  
that policy using perturbed reward average $\bar{Y}_{k,s}^{*}$ as arm index, in spite of improving the Follow-the-Leader algorithm,
still can incur linear regret (Figure.\ref{fig:linear_regret}; orange line). The problem is that
the exploration level of perturbed reward average inherited from current reward samples' statistical uncertainty is not sufficient, leading to under-exploration.

Surprisingly, after we carefully perturb the reward average in an \textit{unconventional} way, we observe that the bandit algorithm successfully secures sublinear regret in the experiment
(Figure.\ref{fig:linear_regret}; green line).
As a summary of our experience, 
we devise a ''regret-optimal`` perturbation scheme and propose it as the \texttt{ReBoot} algorithm.

\subsection{Algorithm \texttt{ReBoot}}\label{sec:alg_reboot}

\begin{algorithm}[t]
  \caption{Residual Bootstrap Exploration (\texttt{ReBoot})}
  \label{alg:residual_bootstrap}
  \begin{algorithmic}[1]
   \REQUIRE Exploration aid unit $\sigma_{a}$
   \FOR{$t = 1, \dots, K$}
   \STATE Pull arm $t$ and get reward $r_t$ \\
   \STATE $\mathcal{H}_{t} \gets (r_{t})$
   \ENDFOR
   \FOR{$t = K+1, K+2, \dots$}
   \FOR{$k = 1, \dots, K$}
   \STATE $s \gets |\mathcal{H}_{k}|$; \quad$\bar{Y}_{k} \gets \text{Average}(\mathcal{H}_{k})$
   \FOR{$i = 1, \dots, s$}
   \STATE $e_{ i} \gets Y_{k,i} - \bar{Y}_{k}$
   \ENDFOR
   \STATE   \label{alg:pseudo_residual}
  $e_{ s+1} \gets \sqrt{s+2}\,\sigma_{a}; \quad e_{ s+2} \gets - \sqrt{s+2}\,\sigma_{a}$
    \STATE Generate bootstrap weights $w_1, \ldots, w_{s+2}$
    \STATE 
    $\hat{\mu}_{k}^{*} \gets \bar{Y}_{k,s} + 
    \text{Average}((w_i e_i)_{i \in [s+2]})
    $
   \ENDFOR
   \STATE Pull arm $I_t \gets \arg\max_{k \in [K]} \hat{\mu}^{*}_k$, and get reward $r_t$
   \STATE $\mathcal{H}_{I_{t}}\gets 
   \mathcal{H}_{I_{t}} \oplus (r_{t})$
   \ENDFOR 
  \end{algorithmic}
\end{algorithm}

We propose the \texttt{ReBoot} algorithm which explores via residual bootstrap with a regret-optimal perturbation scheme. Specifically, at each round, for each arm $k\in[K]$ (denoting by $\mathcal{H}_k=(Y_{k,1},\dots,Y_{k,s})$ the \emph{history} of arm $k$: a vector of all $s$ rewards received so far by pulling arm $k$), \texttt{ReBoot} computes an index for arm $k$ via four steps:
\vspace{-3mm}
\begin{enumerate}[noitemsep]
    \item Compute the average reward $\bar{Y}_{k,s}=s^{-1}\sum_{i=1}^{s}Y_{k,i}$.
    \item Compute the  residuals $\{e_{k,i}\}_{i=1}^{s}$ with $e_{k,i} = Y_{k,i}-\bar{Y}_{k,s}$ for all $i\in[s]$. Appending $e_{k,s+1}=\sqrt{s+2}\,\sigma_{a}$ and $e_{ k, s+2}=-\sqrt{s+2}\,\sigma_{a}$ as two \emph{pseudo} residuals ($\sigma_{a}>0$ is a tuning parameter).\label{step:pseudo}
    \item Generate bootstrap weights (random variables with zero mean and unit variance) $\{w_i\}_{i=1}^{s+2}$.  \label{step:weight}
    \item 
    Add the reward average $\bar{Y}_{k,s}$
with the average of perturbed residuals  
$(s+2)^{-1}\sum_{i=1}^{s+2} w_i e_{k,i}$ to the get \textit{arm index} $\bar{\mu}_{k,s}^{*}$.
\label{step:boots}
\end{enumerate}
\vspace{-2 mm}
Then, \texttt{ReBoot} follows the bootstrap leader. That is, \texttt{ReBoot} pulls the arm with the highest arm index; formally,
\begin{equation}
    I_{t} = \arg\max_{k \in [K]}\hat{\mu}_{k}^{*}.
\end{equation}
A summarized \texttt{ReBoot} algorithm for the MAB problem is presented in Algorithm \ref{alg:residual_bootstrap}.
We explain the residual bootstrap (Step \ref{step:boots}) and discuss the choice of bootstrap weights (Step \ref{step:weight}) in Section \ref{subsec:ReBootExploration_MAB}, and we explain the reason for adding pseudo residuals in Section \ref{subsec:prevent_overunder_exploration}. 

\subsection{Residual bootstrap exploration with regret optimal perturbation scheme}\label{subsec:ReBootExploration_MAB}

In this subsection, we illustrate how to implement residual bootstrap exploration with a proposed regret-optimal perturbation scheme in MAB problem. We give full descriptions on the four steps of the \texttt{ReBoot} with discussion on how the proposed regret-optimal perturbation scheme boosts the uncertainty to escape from suboptimal solutions. 
Before proceeding to the exact description of our proposed policy, it's helpful to introduce further notations for MAB. At round $t$, the \textit{number of pulls} of arm $k$ is denoted by $T_{k,t}\equiv \sum_{j=1}^{t}I(I_{j}=k)$. 

\textbf{Step 1. Compute the average reward of history}

We first describe a \textit{historical reward} vector $\mathcal{H}_{k,s}$ for arm $k$ when $T_{k,t}=s$ at round $t$. Denote the $l$-th entry of $\mathcal{H}_{k,s}$ by $Y_{k,l}$ the reward of arm $k$ received after the $l$-th pull. 

Associated with a historical reward vector $\mathcal{H}_{k,s}$ is the average reward 
$\bar{Y}_{k,s}=s^{-1} \sum_{l=1}^{s}Y_{k,l}$.

\textbf{Step 2. Compute residuals and pseudo residuals}

Given the average reward $\bar{Y}_{k,s}$ of the 
historical reward vector $\mathcal{H}_{k,s}$, compute the residual set $\{e_{k,i} = Y_{k,i}-\bar{Y}_{k,s}\}_{i=1}^{s}$. Note that the residual set carries the statistical uncertainty among the history $\mathcal{H}_{k,s}$, 
contributing to part of exploration level
used in the \texttt{ReBoot}. As we see in Section \ref{sec:vanilla_RB}, such exploration level is not enough to secure sublinear regret.

Encouraging exploration to escape from suboptimal solutions, especially when number of rewards is little, requires a sophisticated design of perturbation scheme. Our proposal on devising a variance-inflated average of perturbed rewards applies the following three steps. First, specify an \textit{exploration aid unit} $\sigma_{a}$, a tuning parameter of \texttt{ReBoot}.
Second, generate  \textit{pseudo} residuals by the scheme
\begin{eqnarray}
e_{k,s+1} &=& \sqrt{s+2}\cdot \sigma_{a}, \label{eq:pseu_plus}\\ 
e_{k,s+2} &=&\sqrt{s+2}\cdot (-\sigma_{a})\label{eq:pseu_minus}.
\end{eqnarray}
Last, append pseudo residuals $\{e_{k,s+1}, e_{k,s+2}\}$ to  $\{e_{k,i}\}_{i=1}^{s}$
to form an augmented residual set $\{e_{k,i}\}_{i=1}^{s+2}$.

\textbf{Step 3. Generate bootstrap weights}

We generate bootstrap weights $w_1,\dots,w_{s+2}$ by drawing i.i.d.\ random variables from a mean zero and unit variance distribution. As recommended in the literature of residual bootstrap \cite{mammen1993bootstrap}, choices of bootstrap weights include 
Gaussian weights, Rademacher weights and skew correcting weights.

\textbf{Step 4. Perturb average reward with residuals}

The arm index $\hat{\mu}_{k,s}^{*}$ using in the \texttt{ReBoot} is then computed by summing up the average reward  $\bar{Y}_{k,s}$
and the average of perturbed augmented residual set $\{w_i \cdot e_{i,s}\}_{i=1}^{s+2}$; formally, 
\begin{equation}\label{eq:reboot_index}
    \hat{\mu}_{k,t}^{*}
    =
    \bar{Y}_{k,s}
    +
    \frac{1}{s+2}\sum_{i=1}^{s+2}w_{i}\cdot e_{k,i}.
\end{equation}

\begin{remark}
Compared to the perturbed average reward  \eqref{eq:perterbed_mean} in vanilla residual bootstrap exploration, the arm index \eqref{eq:reboot_index} used in the \texttt{ReBoot} possesses additional exploration level controlled by the tuning parameter $\sigma_{a}$ in equations \eqref{eq:pseu_plus} and \eqref{eq:pseu_minus}. Intuitively, larger $\sigma_{a}$ delivers stronger exploration assistance for arm index \eqref{eq:reboot_index}, increasing the chance of escaping from suboptimal solutions. 
\end{remark}

\textbf{How the \texttt{ReBoot} explores?}

Here we explain, conditioning on historical reward vector $\mathcal{H}_{k,s}$, how the arm index used in the \texttt{ReBoot} explores. By our perturbation scheme, at round $t$, the arm index $\hat{\mu}_{k,t}^{*}$ admits a presentation
\begin{equation}\label{eq:resample_mean_AEA}
    \hat{\mu}_{k,t}^{*} = \bar{Y}_{k,s} + 
    \frac{1}{s+2}[
    \sum_{i=1}^{s}w_i \cdot e_{k,i}
    +
    \sum_{i=s+1}^{s+2}
    w_i \cdot e_{k,i}
    ],
\end{equation}
where $e_{k,s+1}, e_{k,s+2}$ are pseudo residuals specified from the proposed perturbation scheme(equations \eqref{eq:pseu_plus} and \eqref{eq:pseu_minus}). 
The data-driven exploration, contributed by perturbed residuals $\{w_{i}\cdot e_{k,i}\}_{i=1}^{s}$, reflects the current reward samples' statistical uncertainty; the additional exploration aid, contributed by perturbed pseudo residuals $\{w_{s+1}\cdot e_{k,s+1},w_{s+2}\cdot e_{k,s+2} \}$, echos the expected statistical uncertainty in the scale of specified exploration aid unit $\sigma_{a}$. 
The art is to tune the parameter $\sigma_{a}$ to strike a balance between these two source of exploration, and to avoid underexploration and secure linear regret. (See more discussion on the intuition and interpretation on $\sigma_{a}$ in Section \ref{subsec:prevent_overunder_exploration}).

To sum up, given a historical reward vector $\mathcal{H}_{k,s}$, the conditional variance of arm index \eqref{eq:reboot_index} has a formula
\begin{equation}\label{eq:cond_var_formula}
\text{Var}(\hat{\mu}_{k,t}^{*}|\mathcal{H}_{k,s}) = \frac{1}{(s+2)^2}[\text{RSS}_{k,s} + \text{PRSS}_{s, \sigma_{a}}],
\end{equation}
which consists of the residual sum of squares, $\text{RSS}_{s}$ (see \eqref{eq:RSS}), from the perturbed residuals $\{w_i\cdot e_{k,i}\}_{i=1}^{s}$,
and the \textit{pseudo residual sum of square},
\begin{equation}\label{eq:prss}
    \text{PRSS}_{s, \sigma_{a}} 
    = 2\cdot (s+2)\cdot \sigma_{a}^2,
\end{equation}
 from the perturbed pseudo residuals $\{w_{s+1}\cdot e_{k,s+1}, w_{s+2}\cdot e_{k,s+2}\}$ at the level of tuning parameter $\sigma_{a}$. Later, in formal regret analysis in Section \ref{sec:regret_analysis}, we will see how  $\text{PRSS}_{s, \sigma_{a}}$ can assist the arm index $\hat{\mu}_{k,t}^{*}$ to prevent underexploration, leading to a pathway to a successful escape from suboptimal solutions.

\subsection{Use exploration aid unit to manage residual bootstrap exploration }\label{subsec:prevent_overunder_exploration}

In this subsection, we discuss the intuition and interpretation of the tuning parameter $\sigma_{a}$ in \texttt{ReBoot}.

\textbf{Choice of exploration aid unit $\sigma_{a}$.}

Now, we answer the question of what level of exploration aid is appropriate for the sake of MAB exploration. The art is to choose a level that can prevent the index $\hat{\mu}_{k,s}^{*}$ from underexploration. We first mentally trace such heuristics, and then provide an exact implementation scheme.

Intuitively, we want to choose a exploration aid unit $\sigma_{a}$ such that the $\text{PRSS}_{s,\sigma_{a}}$ in \eqref{eq:cond_var_formula} plays a major role when number of reward samples $s$ is small and then gradually loses its importance. Such consideration is an effort to preserve statistical efficiency of averaging procedure in the proposed regret-optimal perturbation scheme.

\textbf{Manage exploration level via exploration aid unit $\sigma_{a}$.}
Now we showcase our craftsmanship. Set the exploration unit $\sigma_{a}$ at the scale that inflates the reward distribution standard deviation $\sigma$ by a \textit{inflation ratio} $r$ such that \begin{equation}\label{eq:var_infla}
\sigma_{a} = r\cdot \sigma.
\end{equation}
An immediate consequence of scheme \eqref{eq:var_infla} with positive inflation ratio $r$ is a considerable uncertainty boosting on the arm index \eqref{eq:reboot_index}, especially at the beginning of bandit algorithm (the regime of little number of reward samples).

Note that in scheme \eqref{eq:var_infla}, the pseudo residual sum of squares \eqref{eq:prss} admits a formula $\text{PRSS}_{s,\sigma_{a}}
=2\cdot r^2 \cdot (s+2)\cdot \sigma^2$. That is, for a size $s$ reward sample set, pseudo residuals $e_{k,s+1}$ and $e_{k,s+2}$ collectively inflate the arm index variance to $2r^2(s+2)$ times reward distribution variance $\sigma^2$.

Now we illustrate how the exploration aid unit $\sigma_{a}$ manages the exploration of arm index $\hat{\mu}_{k,t}^{*}$ given a historical reward vector $\mathcal{H}_{k,s}$. First, we note that, with a sufficiently large inflation ratio $r$, the data-driven variation $\text{RSS}_{s}$ is \textit{dominated} by the pseudo residual sum of square $\text{PRSS}_{s, \sigma_{a}}$; formally, for any number of reward samples $s$, the good event 
\begin{equation}\label{eq:good_event}
    G_{k,s} = \{\text{RSS}_{k,s} \le \text{PRSS}_{s, \sigma_{a}}\}.
\end{equation} is of high probability. Second, given the good event $G_{k,s}$, formula \eqref{eq:cond_var_formula} implies that the variance of the arm index $\hat{\mu}_{k,t}^{*}$ given a historical reward vector $\mathcal{H}_{k,s}$ stays within a certain \textit{pre-specified} range proportional to $\text{PRSS}_{s,\sigma_{a}}$; that is, 
\begin{equation}
    \frac{\text{PRSS}_{s,\sigma_{a}}}{(s+2)^2}
    \le
    \text{Var}(\hat{\mu}_{k,t}^{*}|\mathcal{H}_{k,s})
    \le 
    2\cdot\frac{\text{PRSS}_{s,\sigma_{a}}}{(s+2)^2}.
\end{equation}
Last, involving the variance inflation scheme \eqref{eq:var_infla} with inflation ratio $r$, $\text{Var}(\hat{\mu}_{k,t}^{*}|\mathcal{H}_{k,s})$ is enclosed in a range proportional to reward distribution variance $\sigma^2$; i.e.
\begin{equation} \label{eq:variance_enclosure}
    \frac{2r^2}{s+2}
    \cdot 
    \sigma^2
    \le
    \text{Var}(\hat{\mu}_{k,t}^{*}|\mathcal{H}_{k,s})
    \le 
    \frac{4r^2}{s+2}
    \cdot \sigma^2.
\end{equation}
The key consequence of two-sided variance bound \eqref{eq:variance_enclosure} on $\hat{\mu}_{k,t}^{*}|\mathcal{H}_{k,s}$ is that, on good event $G_{k,s}$, the index of arm $k$ given $s$ reward samples does not suffer either severe overexploration or underexploration. Then, we are able to avoid severe underestimation for the optimal arm and severe overestimation for the suboptimal arms.

\textbf{Practical choice of inflation ratio $r$.}

For practice, we recommend to choose the inflation ratio $r = 1.5$ in MAB with unbounded reward. This choice is supported theoretically by formal analysis of Gaussian bandit presented in Regret Analysis section (Section \ref{sec:regret_analysis})
and 
empirically by experiments including Gaussian, exponential and logistic bandits in Figure \ref{fig:exp2} in Section \ref{sec:experiments}.

\begin{remark}
All treatments above did not impose distributional assumptions (i.e., the shape of distribution) on the arm reward (can be bounded or unbounded) and bootstrap weights (only assume zero mean and unit variance). In Section \ref{sec:comp}, we discuss the benefits of Gaussian bootstrap weight in \texttt{ReBoot}, leading to an efficient implementation  
with storage and computation cost as low as Thompson sampling.
\end{remark}

\subsection{Efficient implementation using Gaussian weight} \label{sec:comp}

A significant advantage of choosing \textit{Gaussian} bootstrap weight is low storage and
computational cost. This is due to the resulting conditional normality of arm index \eqref{eq:reboot_index}, \textit{no matter} the underlying reward distribution. The arm index of the \texttt{ReBoot} under Gaussian bootstrap weight condition on the historical reward is Gaussian distributed with sample average $\bar{Y}_{k,s}$ as its mean parameter and $s^{-2}\text{RSS}_{k,s}$
as its variance parameter. That is, 
\begin{equation}
    \hat{\mu}_{k,t}^{*}|\mathcal{H}_{k,s}
    \sim 
    N(\bar{Y}_{k,s}, s^{-2}\text{RSS}_{k,s}).
\end{equation}

Then, it can be implemented efficiently by the following incremental updates. At round $s$, after pulling arm $k$, we update $\bar{Y}_{k,s}$ $S_{a,s}\equiv\sum_{i=1}^s Y_{k,i}^2$, and $\text{RSS}_{k,s}$ by
\begin{equation}
    \bar{Y}_{k,s}=[(s-1)\bar{Y}_{k,s-1}+Y_{k,s}]/s ,\quad S_{k,s}=S_{k,s-1}+Y_{k,s}^2,
\end{equation}
and thus, $\text{RSS}_{k,s}=S_{k,s}-s\bar{Y}_{k,s}^2$. $\hat{\mu}_{k}^{*}$ in \eqref{eq:resample_mean_AEA} can be computed by a similar efficient approach. This implementation yielded by Gaussian bootstrap weight saves both storage and computational cost and makes \texttt{ReBoot} as efficient as \texttt{TS} \cite{agrawal2013further}. We compare \texttt{ReBoot} with \texttt{TS}, \texttt{Giro} \cite{kveton2018garbage}, and \texttt{PHE} \cite{kveton2019perturbed} on storage and computational cost in Table \ref{tab:storage}. An empirical comparison on computational cost is done in Section \ref{sec:exp3}.

\begin{table}[t!]
\caption{Storage and computational cost of \texttt{ReBoot}, \texttt{TS}, \texttt{Giro}, and \texttt{PHE}.}
\label{tab:storage}
\vskip 0.15in
\begin{center}
\begin{small}
\begin{tabular}{ccc}
\toprule
Algorithm    & Storage & Computational cost \\
\midrule
\texttt{TS}      & $O(K)$ & $O(KT)$ \\
\texttt{Giro}      & $O(aT)$ & $O(aT^2)$ \\
\texttt{PHE}      & $O(K)$ & $O(KT)$ \\
\texttt{ReBoot} & $O(K)$ & $O(KT)$ \\
\bottomrule
\end{tabular}
\end{small}
\end{center}
\vskip -0.1in
\end{table}

\section{Regret Analysis} \label{sec:regret_analysis}

\subsection{Gaussian \texttt{ReBoot}}

We analyze \texttt{ReBoot} in a K-armed Gaussian bandit. The setting and regret are defined in Section \ref{sec_ReBoot_alg}. We further assume the reward distribution of arm $k$ is Gaussian distributed with mean $\mu_k$ and variance $\sigma^2 = 1$.

\begin{theorem}\label{thm:GaussianMAB_main}
Consider a K-armed Gaussian bandit where the reward distribution of arm k is drawn from Gaussian distribution $N(\mu_k, 1)$.
Let $\sigma_{a}>1.5$ be the exploration aid unit.
Then, the T round regret of \texttt{ReBoot} satisfies: 
\begin{equation}
R(T)
\le 
\sum_{k=2}^{K}\Delta_{k}
[6 
+ 
\{
C_1(\sigma_{a}) + C_2(\sigma_{a}) 
\Delta_{k}^{-2}
\}
\cdot 
\log T]
,
\end{equation}
where the constants $C_1(\sigma_{a})$ and $C_2(\sigma_{a})$ are defined as 
\begin{eqnarray}
C_1(\sigma_{a}) &=& 8(2\sigma_{a}^2-1)^{-1},\label{constant_C1}\\
C_2(\sigma_{a}) &=& 128\sigma_{a}^2
\big(
3.1 + 2(1-2.25 \sigma_{a}^{-2})^{-\frac{1}{2}}
\big)\label{constant_C2}.
\end{eqnarray}
\end{theorem}
\begin{proof}
We defer the proof and rigorous non-asymptotic analysis to Appendix \ref{Appe:main_regret}. The key steps and asymptotic reasoning are presented in subsection \ref{sec:proof_scheme_main_thm}.
\end{proof}
After further optimizing the constants and assuming, without loss of generality, the maximum suboptimality gap $\max_{k \in [K]}\Delta_{k} < 1$, we have the following corollary.
\begin{corollary}\label{Cor:reboot_opt}
Choose $\sigma_{a}=1.7$ in Theorem \ref{thm:GaussianMAB_main}. Then, the T round regret of \texttt{ReBoot} satisfies:
\begin{equation}
R(T)\lesssim 
\sum_{k=2}^{K}\Delta_k 
+
\sum_{k=2}^{K}
\frac{ \log T}{\Delta_k}.
\end{equation}
\end{corollary}

\begin{remark}
Corollary \ref{Cor:reboot_opt} demonstrates that the regret bound of the proposed \texttt{ReBoot} algorithm matches the state-of-art theoretical result for MAB based on UCB algorithm (Theorem 7.1 in \citet{lattimore2018bandit}).
\end{remark}

\textbf{Compare to \texttt{Giro}.}
The way of adding pseudo observations to analyze regret of Bernoulli bandit in \citet{kveton2018garbage} heavily relies on the \textit{bounded support} assumption on reward distribution.
Our theoretical contribution is to carry the regret analysis beyond bounded support reward distribution  to unbounded reward distribution regime, by introducing \textit{residual perturbation}-based exploration in MAB problem.

\textbf{Technical novelty compared to \texttt{Giro}.} 
The argument in \texttt{Giro} for proving regret upper bound does not directly apply to Gaussian bandit because they rely on the fact that the sample variance of Bernoulli reward is bounded. Indeed, in Gaussian bandit, the sample variance has chi-square distribution, which is not bounded. We overcome such predicament after recognizing a consequential good event of proposed perturbation scheme. Certainly, our novel regret-optimal perturbation scheme cages the exploration level of arm index into a two-sided bound with high probability in Gaussian bandit. Such bound is controllable by the tuning parameter $\sigma_{a}$ of \texttt{ReBoot} and capable of preventing underexploration phenomenon of vanilla residual bootstrap exploration.

\subsection{Discussion on choosing exploration aid unit $\sigma_{a}$}

The condition $\sigma_{a}>1.5$ is to ensure the constant $C_{2}(\sigma_{a})$ in \eqref{constant_C2} is finite. Constant $C_{2}(\sigma_{a})$ comes from analysis of $a_k$, i.e., the expected number of sub-optimal pulls due to underestimation on the optimal arm. Large $\sigma_{a}$ helps with jumping off the bad instance where the reward samples of the optimal arm is far below its expectation. Constant $C_2(\sigma_{a})$ in \eqref{constant_C2} is  decreasing in $\sigma_{a}$ for $\sigma_{a} > 1.7$. Constant $C_1(\sigma_{a})$ in \eqref{constant_C1} is decreasing in $\sigma_{a}$ for $\sigma_{a} > 0.5$. Therefore, in Corollary \ref{Cor:reboot_opt}, we pick $\sigma_{a} = 1.7$ to optimize the constant. Since \texttt{ReBoot} performs well empirically, as we show in Section \ref{sec:experiments}, the theoretically suggested value of exploration aid unit $\sigma_{a}$ is likely to be loose.

\subsection{Proof Scheme}\label{sec:proof_scheme_main_thm}

We roadmap the proof scheme of Theorem \ref{thm:GaussianMAB_main}. The key is to analyze the situation that leads to pulling a sub-optimal arm. Such situation consists two type of events: underestimating the optimal arm and overestimating a suboptimal arm.

As shown in the Theorem 1 of \cite{kveton2018garbage}, the $T$ round regret of perturbed history type algorithm has an upper bound
\begin{equation}\label{eq:regret_upp_bound}
    R(T) \le \sum_{k=2}^{K}\Delta_{k}(a_{k}+b_{k}).
\end{equation}
The first term $a_{k}$ is the expected number of rounds that the optimal arm 1 has been being underestimated; formally, 
\begin{equation}
    a_{k} = \sum_{s=0}^{T-1}
    E[\min\{N_{1,s}(\tau_k), T\}],
\end{equation}
where $N_{1,s}(\tau_k)$ is the expected number of rounds that the optimal arm 1 being underestimated given $s$ sample rewards. The second term $b_{k}$ is the probability that the suboptimal arm $k$ is being overestimated; formally, 
\begin{equation}
    b_{k} = 1+ \sum_{s=0}^{T-1}P(Q_{k,s}(\tau_k)>T^{-1}),
\end{equation}
where $Q_{k,s}(\tau_k)$ is the probability of the suboptimal arm k is being overestimated given $s$ sample rewards.

Here we explain the situation that the bandit algorithm \textit{will not} pull the suboptimal arm $k$ at round $t$. Consider the optimal arm 1 and a suboptimal arm $k$. At round $t$, suppose $T_{1,t} = s_1$ and $T_{k,t} = s_k$, then the indexes of arm 1 and arm $k$ are $\hat{\mu}_{1,s_1}^{*}$ and $\hat{\mu}_{k,s_k}^{*}$.
Given a constant level $\tau_k \in \mathbb{R}$, we define the event of \textit{underestimated the optimal arm 1} as 
\begin{equation}
    F_{s_1} = \{\hat{\mu}_{1,s_1}^{*} \le \tau_k\}
\end{equation}
and the event of \textit{overestimated a suboptimal arm k} as
\begin{equation}
    E_{s_k}^{c} = \{\hat{\mu}_{k,s_k}^{*} > \tau_k\}. 
\end{equation}
If we pick $\tau_k \in (\mu_k, \mu_1)$ and the distribution of indexes both have exponential decaying tails, theory of large deviation indicates that the events of $F_{s_1}$ and $E_{s_k}^{c}$ both are rare events asymptotically. Given both events $F_{s_1}$ and $E_{s_k}^{c}$ happens, the agent \textit{will not} pull the suboptimal arm k.

\textbf{Roadmap of bounding $a_k$}

We provide asymptotic reasoning on bounding $a_k$ and defer the non-asymptotic analysis to lemma \ref{lm:a_ks}. 
Recall that for a given constant level $\tau_k$, the probability of the optimal arm 1 being underestimated given $s$ reward samples is $1-Q_{1,s}(\tau_{k})$. If we pick the level to satisfy $\tau_k < \mu_1$, the theory of large deviation gives
\begin{equation}\label{eq:optarm_asymp}
    Q_{1,s}(\tau_k) \overset{s\to \infty}{\to} 1.
\end{equation}
Recall that $N_{1,s}(\tau_k) (= \frac{1}{Q_{1,s}(\tau_k)}-1)$ is the expected number of rounds to observe a not-under-estimated instance from resample mean distribution $\hat{\mu}_{1,s}$ given $s$ reward samples. The asymptotics in \eqref{eq:optarm_asymp} implies $N_{s}(\tau_k) \to 0$ as the number of pulls $s$ grows to infinity.
Thus, given the time horizon $T$, there exists a constant $s_0(T)$ such that 
$N_{s}(\tau_k) \le T^{-1}$
for all $s$ over $s_0(T)$. Consequently, the quantity $a_k$ in regret bound \eqref{eq:regret_upp_bound} is bounded by 
\begin{equation}
a_k 
\le 
1+\sum_{s=0}^{s_0(T)}E[\min\{N_{1,s}(\tau_k), T\}].
\end{equation}
The fact that constant $s_0(T)$ is of $O(\log T)$ order will be shown in lemmas \ref{lm:a_ks1}
and  \ref{lm:a_ks2}. For small number of pull $s < s_0(T)$, we show in lemma \ref{lm:a_ks} that $[\min\{N_{1,s}(\tau_k), T\}] \le 1.1 + (1-\frac{9}{4}\frac{\sigma^2}{\sigma_a^2})^{-\frac{1}{2}}$ for any $s$. Thus, it is enough to conclude that $a_k$ can be bounded by a term of $O(\log T)$ order.

\textbf{Roadmap of bounding $b_k$}

We provide asymptotic reasoning on bounding $b_k$ and defer the non-asymptotic analysis to lemma \ref{lm:b_ks1}. Recall that for a given constant level $\tau_k$, the probability of the  suboptimal arm $k$ being overestimated given $s$ reward samples is $Q_{k,s}(\tau_{k})$. If we pick the level to satisfy $\tau_k > \mu_k$, the theory of large deviation gives
\begin{equation}\label{eq:suboptarm_asymp}
    Q_{k,s}(\tau_k) \overset{s\to \infty}{\to} 0.
\end{equation}
Thus, given the time horizon $T$, there exists a constant $s_0(T)$ such that $Q_{k,s}(\tau_k) \le T^{-1}$ for all $s$ over $s_0(T)$. As a result, the event $\{Q_{k,s}(\tau_k) >\frac{1}{T}\}$ is empty if the number of pull $s$ is beyond $s_0(T)$.
Consequently,
the quantity $b_k$ in regret bound \eqref{eq:regret_upp_bound} is bounded by 
\begin{equation}
    b_k \le \sum_{s=0}^{s_0(T)}P(Q_{k,s}(\tau_k) >T^{-1})
\end{equation}
The fact that constant $s_0(T)$ is of $O(\log T)$ order will be shown in lemmas \ref{lm:b_ks2} and \ref{lm:b_ks3}. For small number of pull $s < s_0(T)$, we apply trivial bound $P(Q_{k,s}(\tau_k) >T^{-1})\le 1$ that holds for any $s$. Therefore, it is enough to conclude that $b_k$ can be bounded by a term of $O(\log T)$ order.

\section{Experiments} \label{sec:experiments}

We compare \texttt{ReBoot} to three baselines: \texttt{TS} \cite{agrawal2013further} (with $\mathcal N(0,1)$ prior), \texttt{Giro} \cite{kveton2018garbage}, and \texttt{PHE} \cite{kveton2019perturbed}. For all the experiments unless otherwise specified, we choose $a=1$ and $a=2.1$ for \texttt{Giro} and \texttt{PHE} respectively, as justified by the associated theory. All the results are averaged over $500$ runs.

\begin{figure}[t!]
\begin{center}
\centerline{\includegraphics[width=\columnwidth]{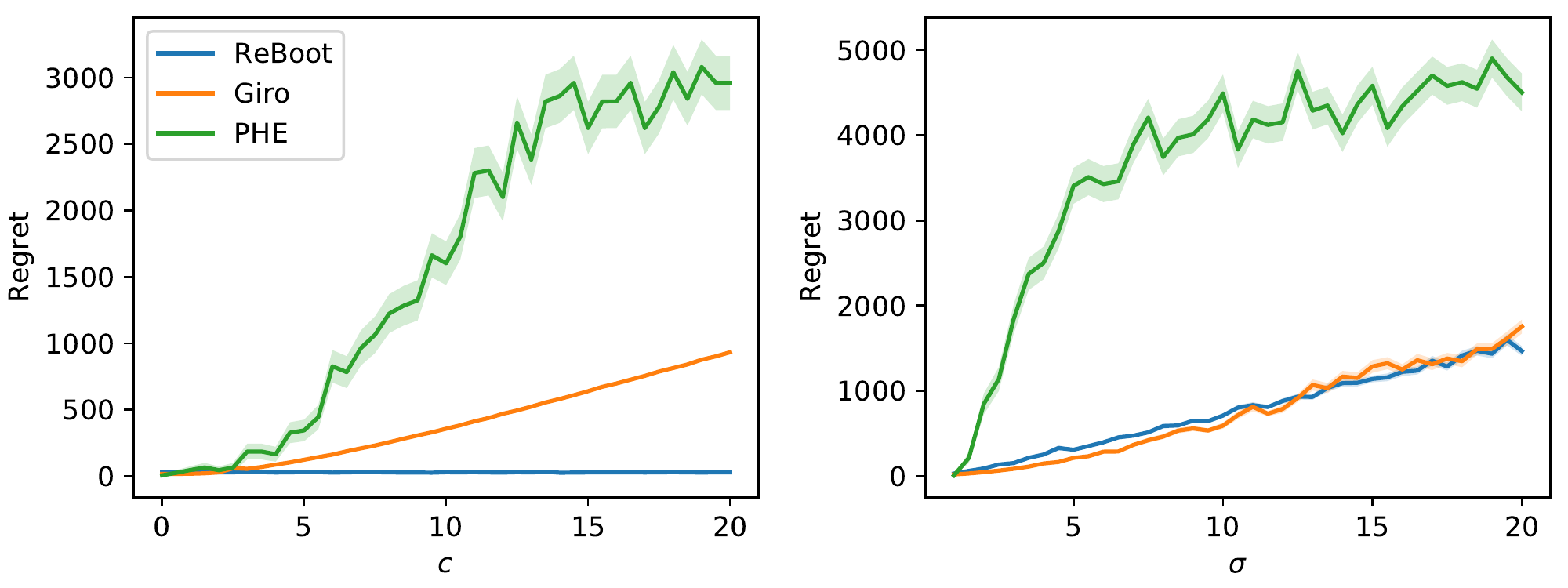}}
\caption{Comparison of \texttt{ReBoot} to \texttt{Giro} and \texttt{PHE} in $10{,}000$-round regret on Gaussian bandits with shifted mean (\textbf{left}) and varying variance (\textbf{right}). The error bars represent the standard error of the mean regret.}
\label{fig:exp1}
\end{center}
\vskip -0.4in
\end{figure}

\subsection{Robustness to Reward Mean/Variance}

\begin{figure*}[]
\vskip 0.2in
\begin{center}
\centerline{\includegraphics[width=\textwidth]{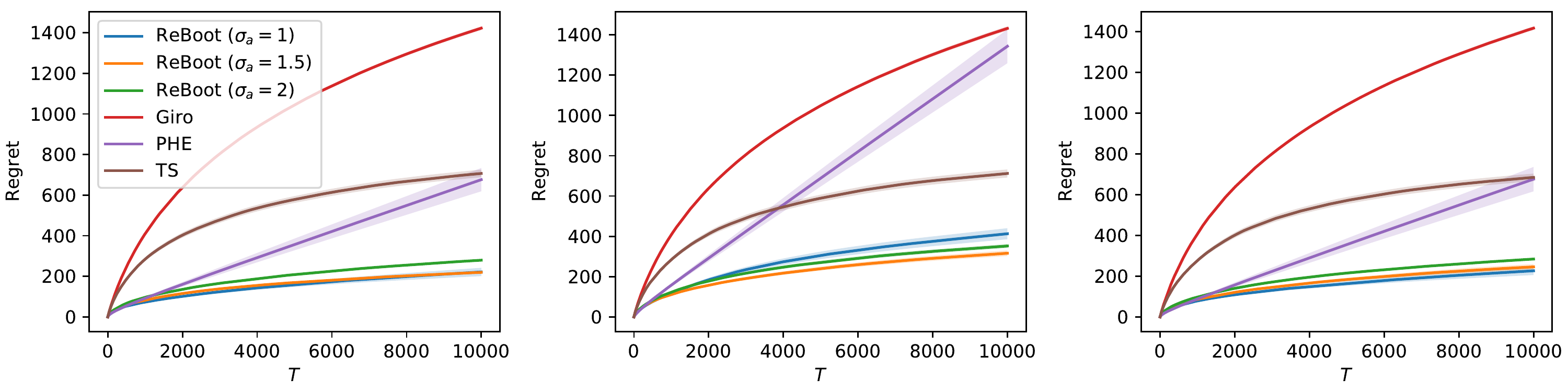}}
\caption{Comparison of \texttt{ReBoot} to \texttt{Giro} and \texttt{PHE} in the regret of the first $10{,}000$ rounds on Gaussian (\textbf{left}), exponential (\textbf{middle}), and logistic (\textbf{right}) bandits. The error bars represent the standard error of the mean regret.}
\label{fig:exp2}
\end{center}
\vskip -0.2in
\end{figure*}

We compare \texttt{ReBoot} with the two bounded bandit algorithms, \texttt{Giro} and \texttt{PHE}, on two classes of $2$-armed Gaussian bandit problems where $P_k = \mathcal N(\mu_k, \sigma)$, $k=1,2$. The first class has $(\mu_1,\mu_2)=(c+1,c)$ and $\sigma=1$, where $c$ varies from $0$ to $20$. The second class has $(\mu_1,\mu_2)=(1,0)$ and $\sigma$ varying from $1$ to $20$, where we choose $\sigma_a=1.5\sigma$ as guided by our theory. Figure \ref{fig:exp1} shows the effect of the shifted mean and the varying variance in three algorithms on the $10{,}000$-round regret.

From the left panel of Figure \ref{fig:exp1}, we see that \texttt{ReBoot} is robust to the increase in the mean rewards, while both \texttt{Giro} and \texttt{PHE} are sensitive. The reason is that when the mean rewards increase, the added pseudo rewards ($\{0,1\}$ for \texttt{Giro} and $\text{Ber}(0.5)$ for \texttt{PHE}) cannot represent the upper extreme value which is supposed to help with escaping from sub-optimal arms. From the right panel of Figure \ref{fig:exp1}, we observe a slow growth in the regret of \texttt{ReBoot} and \texttt{Giro} as the variance increases due to the raised problem difficulty level, while \texttt{PHE} is much more sensitive than \texttt{ReBoot} and \texttt{Giro} since the exploration in \texttt{PHE} completely relies on pseudo rewards which are inappropriate in varying variance but \texttt{ReBoot} and \texttt{Giro} mostly depends on bootstrap which yields more stable performance.

\subsection{Robustness to Reward Shape} \label{sec:exp3}

We compare \texttt{ReBoot} with \texttt{TS}, \texttt{Giro}, and \texttt{PHE} on three classes of $10$-armed bandit problems:
\begin{itemize}[noitemsep]
    \item $P_k = \mathcal N(\mu_k, 1)$ with $\mu_k\sim\text{Unif}(5,7)$;
    \item $P_k = \text{Exp}(\mu_k) + 5$ with mean $\mu_k\sim\text{Unif}(0,2)$;
    \item $P_k$ is a logistic distribution with mean $\mu_k\sim\text{Unif}(5,7)$ and variance $1$.
\end{itemize}
In each class, the mean reward takes values in $[5,7]$, and the variance is either $1$ for all arms (Gaussian and logistic) or varies in $(0,4]$ among arms (exponential). The regret of the first $10{,}000$ rounds is displayed in Figure \ref{fig:exp2}.

\texttt{ReBoot} has sub-linear and small regret when $\sigma_a=1.5$ in all cases, which validates our theory for Gaussian bandits and potential applicability in bandits of other distributions with even heteroscedasticity. We also see that $\sigma_a=1.5$ is a near-optimal choice for Gaussian bandits and exponential bandits (for logistic bandits, $\sigma_a=1$ is slightly better). The linear regret of \texttt{PHE} and the sub-linear but large regret of \texttt{Giro} in all cases are because the mean rewards are shifted away from $[0,1]$. \texttt{TS} cannot achieve its optimal performance without setting the prior with accurate knowledge of the reward distribution.

\begin{table}[t]
\caption{Run times of \texttt{ReBoot}, \texttt{TS}, \texttt{Giro}, and \texttt{PHE}.}
\label{tab:comp}
\vskip 0.15in
\begin{center}
\begin{small}
\begin{tabular}{ll|llll}
\toprule
\multicolumn{2}{c|}{Model} & \multicolumn{4}{c}{Run time (seconds)} \\
$K$           & $T$           & \texttt{TS}   & \texttt{Giro}   & \texttt{PHE}   & \texttt{ReBoot}  \\
\midrule
5            & 1k            & 0.030     & 0.174       &    0.051   & 0.029               \\
10            & 1k            & 0.030     & 0.280       & 0.052      & 0.029               \\
20            & 1k            & 0.030     & 0.491       & 0.052      & 0.029               \\
5            & 10k            & 0.330     & 3.903       & 0.541      & 0.344               \\
10            & 10k            & 0.326     & 4.979       & 0.545      & 0.319               \\
20            & 10k            & 0.325      & 7.065       & 0.551      & 0.315              \\
\bottomrule
\end{tabular}
\end{small}
\end{center}
\vskip -0.1in
\end{table}

\subsection{Computational Cost}

We compare the run times of \texttt{ReBoot} ($\sigma_a=1.5$), \texttt{TS}, \texttt{Giro}, and \texttt{PHE} in a Gaussian bandit. The settings consist of all combinations of $K\in\{5,10,20\}$ and $T\in\{1{,}000,10{,}000\}$. Our results are reported in Table \ref{tab:comp}. In all settings, the run times of \texttt{ReBoot}, \texttt{TS}, and \texttt{PHE} are all comparable, while the run time of \texttt{Giro} is significantly higher due to computationally expensive sampling with replacement over history with pseudo rewards. This comparison validates our analysis in Section \ref{sec:comp}.

\section{Conclusion}

In this work, we propose a new class of algorthm \texttt{ReBoot}: residual bootstrap based exploration mechanism. We highlight the limitation of directly using statistical bootstrap in bandit setting and develop a remedy procedure called variance inflation. We analyze \texttt{ReBoot} in an unbounded reward showcase (Gaussian bandit) and prove an optimal instance-dependent regret.

\newpage

\bibliography{example_paper}
\bibliographystyle{icml2020}
\clearpage

\onecolumn
\appendix
\begin{center}
    \Large Supplement to ``Residual Bootstrap Exploration for Bandit Algorithms''
\end{center}


Section \ref{Appe:main_regret} gives the proof of main regret bound (Theorem \ref{thm:GaussianMAB_main}). Section \ref{Appe:tech_lemmas} gives all technical lemmas required to bound the regret in section \ref{Appe:main_regret}. Section \ref{Appe:supp_lemmas} lists all supporting lemmas, including lower bound of Gaussian tail and concentration bound of chi-square distribution.

\section{Proof of Theorem \ref{thm:GaussianMAB_main}.} \label{Appe:main_regret}

\paragraph{Step 0: Notation and Preparation}

We restate the Theorem 1 in \cite{kveton2018garbage}:
\begin{theorem}\label{GiroThm1} 
Let $Q_{k,s}(\tau)=P(\hat{\mu}^{*}_{k,t}>\tau|\mathcal{H}_{1,s})$ for each arm $k \in [K]$.
For any $\tau_{k} \in \mathbb{R}$, the expected $T-$round regret of \textit{General randomized exploration} algorithm is bounded from above as 
\begin{equation}
    R(T)\le \sum_{k=2}^{K}\Delta_{k}(a_k+b_k),
\end{equation}
where
\begin{equation}
a_{k} = \sum_{s=0}^{T-1}E[\min\{Q_{1,s}(\tau_k)^{-1}-1, T\}]~~~~;~~~~
b_{k} =
\sum_{s=0}^{T-1}P(Q_{k,s}(\tau_k) > T^{-1})+1.
\end{equation}
\end{theorem}
We then use Theorem \ref{GiroThm1} to analyze the regret of Gaussian \texttt{ReBoot}.
\paragraph{Step 1: Bounding $a_k$.}
Recall $\bar{Y}_{1,s} = s^{-1}\sum_{i=1}^{s}Y_{1,i}$ and $\text{RSS}_{1,s}=\sum_{i=1}^{s}(Y_{1,i}-\bar{Y}_{1,s})^2$. 
We define event $A_{1,s} = \{\bar{Y}_{1,s}-\mu_1>-\Delta_{k}/4\}$ and $G_{1,s}=\{\text{RSS}_{1,s} \le \text{PRSS}_{s,\sigma_{a}}\}$. Let $r = \sigma_{a}/\sigma$ and hence $\text{PRSS}_{s,\sigma_{a}} = 2(s+2)r^2\sigma^2$.
Define $Q_{1,s}(\tau)\equiv P(\hat{\mu}^{*}_{1,t}>\tau|\mathcal{H}_{1,s})$ and 
set $N_{1,s}(\tau) = Q^{-1}_{1,s}(\tau)-1$. 
Define $\tau_k =  (\mu_1+\mu_k)/2$.
Set $a_{k,s} = E[\min\{N_{1,s}(\tau_k),T\}]$. Note that by Lemma \ref{lm:a_ks}, $a_{k,s} \le
M(r)
\equiv 
1.1
+(1-(3/(2r))^2)^{-1/2}$ for any $s>0$. Write $a_{k,s} = a_{k,s,1}+a_{k,s,2}+a_{k,s,3}$, where
\begin{eqnarray}
a_{k,s,1}&=&E[\min\{N_{1,s}(\tau_k),T\}I(A_{1,s})I(G_{1,s})],\label{a_ks1}\\
a_{k,s,2}&=&E[\min\{N_{1,s}(\tau_k),T\}I(A_{1,s}^{c})I(G_{1,s})], \label{a_ks2}\\
a_{k,s,3}&=&E[\min\{N_{1,s}(\tau_k),T\}I(G_{1,s}^{c})]. \label{a_ks3}
\end{eqnarray}

From Lemmas \ref{lm:a_ks1}, \ref{lm:a_ks2} and 
\ref{lm:a_ks3}, one has that, for any $s > \max\{s_{a,1}(T),s_{a,2}(T),s_{a,3}(T)\}$, $a_{k,s}\le 3T^{-1}$. Now, set $r>2^{-1/2}$, then $\max\{s_{a,1}(T), s_{a,2}(T), s_{a,3}(T)\} 
= 16 \max\{16(\sigma/\Delta_k)^2r^2, (2r^2-1)^{-1}\}\cdot \log T$.
\begin{eqnarray}
    a_{k}
    &\le& 
    M(r)\cdot 
    \max\{s_{a,1}(T), s_{a,2}(T), s_{a,3}(T)\}
    + (T-\max\{s_{a,1}(T), s_{a,2}(T), s_{a,3}(T)\})\cdot 3T^{-1}\\
    &\le&
    3+ 16M(r) \max\{16(\sigma/\Delta_k)^2r^2, (2r^2-1)^{-1}\}\cdot \log T.
\end{eqnarray}

\paragraph{Step 2: Bounding $b_k$.}

Recall that $\bar{Y}_{k,s} = s^{-1}\sum_{i=1}^{s}Y_{k,i}$ and $\text{RSS}_{k,s}=\sum_{i=1}^{s}(Y_{k,i}-\bar{Y}_{k,s})^2$.
We define events $A_{k,s} = \{\bar{Y}_{k,s}-\mu_k<\Delta_{k}/4\}$ and $G_{k,s}=\{\text{RSS}_{k,s} \le \text{PRSS}_{s,\sigma_{a}}\}$. Let $r = \sigma_{a}/\sigma$ and hence $\text{PRSS}_{s,\sigma_{a}} = 2(s+2)r^2\sigma^2$.

Set $b_{k,s}=E[I(Q_{k,s}(\tau_k)>T^{-1})]$. Note the trivial bound $b_{k,s} \le 1$ for any $s>0$, which follows from the fact that $b_{k,s}$ is a probability.
Write $b_{k,s}=b_{k,s,1}+b_{k,s,2}+b_{k,s,3}$, where
\begin{eqnarray}
b_{k,s,1}&=&
E[I(Q_{k,s}(\tau_k)>T^{-1})I(A_{k,s})I(G_{k,s})],\label{b_ks1}\\
b_{k,s,2}&=&
E[I(Q_{k,s}(\tau_k)>T^{-1})I(A_{k,s}^{c})I(G_{k,s})],\label{b_ks2}\\
b_{k,s,3}&=&
E[I(Q_{k,s}(\tau_k)>T^{-1})I(G_{k,s}^{c})].\label{b_ks3}
\end{eqnarray}
From Lemmas \ref{lm:b_ks1}, \ref{lm:b_ks2} and 
\ref{lm:b_ks3}, one has that, for any $s > \max\{s_{b,1}(T),s_{b,2}(T),s_{b,3}(T)\}$, $b_{k,s}\le 2T^{-1}$. 
Now, set $r > 2^{-1/2}$, then $\max\{s_{b,1}(T), s_{b,2}(T), s_{b,3}(T)\} 
= 8 \max\{16(\sigma/\Delta_k)^2r^2, (2r^2-1)^{-1}\}\cdot \log T$.
\begin{eqnarray}
    b_{k}
    &\le& 1+
    \max\{s_{b,1}(T), s_{b,2}(T), s_{b,3}(T)\}
    + (T-\max\{s_{b,1}(T), s_{b,2}(T), s_{b,3}(T)\})\cdot 2T^{-1}\\
    &\le&
    3+ 8 \max\{16(\sigma/\Delta_k)^2r^2, (2r^2-1)^{-1}\}\cdot \log T.
\end{eqnarray}

\paragraph{Step 3: Bounding $T$-round regret $R(T)$.}
Combine the results we obtained at Step 1 and 2, one has
\begin{equation}
    a_{k}+b_{k}\le 
    6 + (8+16 M(r)) \max\{16(\sigma/\Delta_k)^2r^2, (2r^2-1)^{-1}\}\cdot \log T.
\end{equation}
Then put into Theorem \ref{GiroThm1} to have the claimed bound.
\section{Technical Lemmas}
\label{Appe:tech_lemmas}
\subsection{Lemmas on bounding $a_k$.}

\begin{lemma}[Bounding $a_{k,s}$ for any $s>0$]\label{lm:a_ks} Set $r>3/2$.
For any $s>0$, 
\begin{equation}
    a_{k,s}\le M(r)\equiv 1.1+(1-(3/2r)^2)^{-1/2}.
\end{equation}
\end{lemma}
\begin{proof}

Recall that $r = \sigma_{a}/\sigma$,  $\text{RSS}_{1,s}=\sum_{i=1}^{s}(Y_{1,i}-\bar{Y}_{1,s})^2$, $\text{PRSS}_{s, \sigma_{a}} = 2(s+2)\sigma_{a}^2=2(s+2)r^2\sigma^2$ and $\text{Var}(\hat{\mu}_{1,t}^{*}|\mathcal{H}_{1,s})=(s+2)^{-2}[\text{RSS}_{1,s}+\text{PRSS}_{s, \sigma_{a}}]$.
Note $[s^{-1}\sigma^2]^{-1/2}(\mu_1 - \bar{Y}_{1,s}) \sim N(0,1)$ (From Gaussian reward) and 
$[\text{Var}(\hat{\mu}_{1,t}^{*}|\mathcal{H}_{1,s})]^{-1/2}(\hat{\mu}_{1,t}^{*}-\bar{Y}_{1,s})| \mathcal{H}_{1,s} \sim N(0,1)$ (From Gaussian weight with proposed perturbation scheme). Set $g( \mu) 
= [\text{Var}(\hat{\mu}_{1,t}|\mathcal{H}_{1,s})]^{-1/2}(\mu-\bar{Y}_{1,s})$. Then we have $g(\hat{\mu}_{1,t}^{*})| \mathcal{H}_{1,s} \sim N(0,1)$. Set $h(\mu) = [\sigma^2/s]^{-1/2}(\mu_1 - \mu)$. Then $h(\bar{Y}_{1,s}) \sim N(0,1)$.

\paragraph{Step 1: Reduce to the tail of bootstrap mean of optimal arm\\}

Without loss of generality, assume arm 1 is optimal.
From $\tau_{k} \le \mu_1$, we have 
\begin{equation}
E[\min\{N_{1,s}(\tau_k),T\}]
\le E[Q_{1,s}(\tau_k)^{-1}]
\le E[P(\hat{\mu}_{1,t}^{*} > \tau_{k}|\mathcal{H}_{1,s})^{-1}]
\le 
E[P(\hat{\mu}_{1,t}^{*} > \mu_1|\mathcal{H}_{1,s})^{-1}]
\label{eq:BasicReduction}
\end{equation}

\paragraph{Step 2: Reduce conditional tail probability to the quantity $g(\hat{\mu}^{*}_{1,t})$\\}
We first analyze the tail  $Q_{1,s}(\mu_1) = P(\hat{\mu}_{1,t}^{*} > \mu_1|\mathcal{H}_{1,s})$. Based on normalization procedure, we have
\begin{equation}\label{eq:normalization_optarm}
P(\hat{\mu}_{1,t}^{*} \ge \mu_1|\mathcal{H}_{1,s})
=
P(g(\hat{\mu}_{1,t}^{*})\ge g(\mu_1)|\mathcal{H}_{1,s})
\end{equation}

We find an upper level of the cutoff point $g(\mu_1)$ through
\begin{eqnarray}
    g(\mu_1)
    &\le& 
    [(s+2)^{-2}\text{PRSS}_{s,\sigma_{a}}]^{-1/2}(\mu_1-\bar{Y}_{1,s})
    =
    [2sr^2/(s+2)]^{-1/2}
    [\sigma^2/s]^{-1/2}(\mu_1-\bar{Y}_{1,s})\\
    &\le& (3/2)^{1/2}r^{-1}
    h(\bar{Y}_{1,s}),
\end{eqnarray}
where the first inequality follows from $\text{RSS}_{1,s} \ge 0$ for any $s \ge 1$ and the second inequality follows from $(s+2)/(2s)\le 3/2$ for all $s\ge 1$. Put into \eqref{eq:normalization_optarm}, one has
\begin{equation}
P(\hat{\mu}_{1,t}^{*} \ge \mu_1|\mathcal{H}_{1,s})
\ge
P\bigg(g(\hat{\mu}_{1,t}^{*}) \ge [(3/2)^{1/2}r^{-1}]
    h(\bar{Y}_{1,s}) \bigg|\mathcal{H}_{1,s}\bigg).
\end{equation}

\paragraph{Step 3:Reduce the reciprocal of conditional tail probability \\}
Based on a lower bound of Gaussian distribution (Lemma \ref{cor:Gau_lower}) we have, by setting $t=[(3/2)^{1/2}r^{-1}]
    h(\bar{Y}_{1,s})$,
\begin{eqnarray}
P(\hat{\mu}_{1,t}^{*} \ge \mu_1|\mathcal{H}_{1,s})^{-1}
&\le&
P(g(\hat{\mu}_{1,t}^{*}) \ge [(3/2)^{1/2}r^{-1}]
    h(\bar{Y}_{1,s}) |\mathcal{H}_{1,s})^{-1}\\
&\le& 
\begin{cases}
    \exp([\frac{3}{2r}h(\bar{Y}_{1,s})]^2) &\text{on} \{h(\bar{Y}_{1,s}) \ge r\cdot \sqrt{4\pi/3}\}\\
    [1-\Phi(\sqrt{2\pi})]^{-1}
    &\text{on} \{0 < h(\bar{Y}_{1,s}) < r\cdot \sqrt{4\pi/3}\}
    \end{cases}.
\end{eqnarray}
The expectation has an upper bound 
\begin{equation}
E[P(\hat{\mu}_{1,t}^{*} \ge \mu_1|\mathcal{H}_{1,s})^{-1}]
\le E[\exp([\frac{3}{2r}h(\bar{Y}_{1,s})]^2)]
+
 [1-\Phi(\sqrt{2\pi})]^{-1}.
\end{equation}

Since $h(\bar{Y}_{1,s})\sim N(0,1)$, so $h(\bar{Y}_{1,s})^2 \sim \chi^2_1$. From the moment generating function of $\chi_1^2$, and recall from assumption that $r > 3/2$, one has
$$E[\exp([\frac{3}{2r}h(\bar{Y}_{1,s})]^2)]=\psi_{\chi^2_{1}}((3/(2r))^2) = (1-(3/(2r))^2)^{-1/2}.$$
Combining all together, we conclude that
\begin{equation}
a_{k,s}
=
E[\min\{N_{1, s}(\tau_k),T\}]
\le 1.1+ (1-(3/(2r))^2)^{-1/2}.
\end{equation}
\end{proof}

\begin{lemma}[Bounding $a_{k,s,1}$ at \eqref{a_ks1}]\label{lm:a_ks1}
For any $s \ge s_{a,1}(T) \equiv 256 r^2 (\sigma/\Delta_k)^2 \log T$,
\begin{equation}
    a_{k,s,1}\le T^{-1}.
\end{equation}
\end{lemma}
\begin{proof}
We first note that $a_{k,s,1}\le  E[N_{1,s}(\tau_k)I(A_{1,s})I(G_{1,s})]$.

\textbf{Step 1: Reduce $N_{1,s}(\tau_k)$ to $1-Q_{1,s}(\tau_k)$.\\}
The first step is to notice that, if we find $s_{a,1}(T)$ such that $s>s_{a,1}(T)$ implies the event $G_{1,s} \cap A_{1,s}$ and the event $\{[1-Q_{1,s}(\tau_k)]> T^{-2}$\} are mutually exclusive, then 
$s>s_{a,1}(T)$ also implies  the event $G_{1,s} \cap A_{1,s}$ the event
$\{N_{1,s}(\tau_k)<T^{-1}\}$ are mutually exclusive. To see this, note that
$(T^2-1)^{-1}<T^{-1}$ for $T>1$, and hence
\begin{eqnarray*}
\{N_{1,s}(\tau_k)
< T^{-1}\}
&\supseteq&
\{N_{1,s}(\tau_k)
< (T^2-1)^{-1}\}\\
&=&
\{Q_{1,s}(\tau_k)^{-1}<1+(T^2-1)^{-1}\}\\
&=&
\{Q_{1,s}(\tau_k)^{-1}<(1-T^{-2})^{-1}\}\\
&=&
\{Q_{1,s}(\tau_k)>1-T^{-2}\}\\
&=&
\{[1-Q_{1,s}(\tau_k)]< T^{-2}\}.
\end{eqnarray*}

\textbf{Step 2: Find $s_{s,1}(T)$ such that $s>s_{a,1}(T)$ implies $[1-Q_{1,s}(\tau_k)]< T^{-2}$.\\}
Set $\tau_{k} = \mu_1-\Delta_{k}/2$. To find such $s_{a,1}(T)$ in Step 1, we recall  $Q_{1,s}(x)=P(\hat{\mu}^{*}_{1,t}>x|\mathcal{H}_{1,s})$ and hence
\begin{equation}
[1-Q_{1,s}(\tau_k)]=
P(\hat{\mu}^{*}_{1,t}-\bar{Y}_{1,s}<-[\bar{Y}_{1,s}-\tau_k]|\mathcal{H}_{1,s}).\label{eq:a_ks1_1}
\end{equation}
By the definition of the event $G_{1,s}$, one has $\text{Var}(\hat{\mu}^{*}_{1,t}|\mathcal{H}_{1,s})I(G_{1,s})\le 2(s+2)^{-2}\text{PRSS}_{s,\sigma_{a}} = 4r^2(s+2)^{-1}\sigma^2$ and hence equation \eqref{eq:a_ks1_1} on the event $G_{1,s}$ becomes
\begin{equation}
[1-Q_{1,s}(\tau_k)]I(G_{1,s})
\le 
\exp\bigg(-[4r^2(s+2)^{-1}\sigma^2]^{-1}\frac{[\bar{Y}_{1,s}-\tau_k]^2}{2}\bigg)I(G_{1,s}).\label{eq:a_ks1_2}
\end{equation}

We then recast the good event $A_{1,s}$ that the optimal arm $1$ is not seriously understimated
as
\begin{eqnarray*}
A_{1,s} &=& \{\bar{Y}_{1,s}-\mu_1>-\Delta_k/4\}\\
&=&\{\bar{Y}_{1,s}-\tau_k > \Delta_k/4\}\\
&=&
\{[\bar{Y}_{1,s}-\tau_k]^2>\Delta_k^2/16\}\\
&=&
\{[\bar{Y}_{1,s}-\tau_k]^2/2>\Delta_k/32\}\\
&=&
\{[4r^2(s+2)^{-1}\sigma^2]^{-1}([\tau_k - \bar{Y}_{1,s}]^2/2)>128^{-1}r^{-2}(\sigma/\Delta_k)^{-2}(s+2)\}.
\end{eqnarray*} Pick $s_{a,1}(T) = 256 r^2 (\sigma/\Delta_k)^2 \log T$ to have $128^{-1}r^{-2}(\sigma/\Delta_k)^{-2}(s+2) > 2\log T$
for any $s>s_{a,1}(T)$. Therefore, the equation \eqref{eq:a_ks1_2} on the good event $A_{1,s}$ becomes, for any $s>s_{a,1}(T)$,
\begin{equation}
[1-Q_{1,s}(\tau_k)]I(A_{1,s})I(G_{1,s})\le T^{-2}I(A_{1,s})I(G_{1,s}).
\label{eq:a_ks1_3}
\end{equation} 

\textbf{Step 3. Combine Step 2 into Step 1.\\}
The last result in Step 2 with the observation at Step 1 leads to a fact that 
$N_{1,s}(\tau_k)I(A_{1,s})I(G_{1,s})
\le T^{-1}I(A_{1,s})I(G_{1,s})$ for any $s>s_{a,1}(T)$. Thus, we conclude $$a_{k,s,1}\le E[T^{-1}I(A_{1,s})I(G_{1,s})]\le T^{-1}.$$
\end{proof}

\begin{lemma}[Bounding $a_{k,s,2}$ at \eqref{a_ks2}]\label{lm:a_ks2}
For any $s \ge s_{a,2}(T) \equiv 64(\sigma/\Delta_k)^2\log T$,
\begin{equation}
    a_{k,s,2}\le T^{-1}.
\end{equation}
\end{lemma}
\begin{proof}
Note $a_{k,s,2}\le E[T\cdot I(A_{1,s}^{c})]
=TP(\bar{Y}_{1,s}-\mu_1 < -\Delta_{k}/4)$.
From large deviation of \textit{lower} tail of Gaussian reward sample mean, we have $$P(A_{1,s}^{c})
=P(\bar{Y}_{1,s}-\mu_1 < -\Delta_{k}/4)\le \exp\bigg(-[s^{-1}\sigma^2]^{-1}[\Delta_k/4]^2\bigg).$$
Pick $s_{a,2}(T) =64(\sigma/\Delta_k)^2\log T$ to have $s > s_{a,2}(T)$ implies that $P(A_{1,s}^{c})\le T^{-2}$, and hence $a_{k,s,2}\le T^{-1}$.
\end{proof}

\begin{lemma}[Bounding $a_{k,s,3}$ at \eqref{a_ks3}]\label{lm:a_ks3}
For any $s \ge s_{a,1}(T) \equiv 16(2r^2-1)^{-1}\log T$,
\begin{equation}
    a_{k,s,3}\le T^{-1}.
\end{equation}
\end{lemma}
\begin{proof}
We first note $a_{k,s,3}\le E[T\cdot I(G_{1,s}^{c})]
=T\cdot P(\text{RSS}_{1,s}>
\text{PRSS}_{s,\sigma_{a}})$.
From Gaussian reward, the distribution of scaled RSS is chi-square distributed; that is, $\text{RSS}_{1,s}/\sigma^2\sim \chi^2_{s-1}$. 
Also note $\text{PRSS}_{s,\sigma_{a}}/\sigma^2 = 2(s+2)r^2$. From Lemma \ref{cor:chisquare_concentra}, the tail probability of residual sum of square is exponentially decaying as sub-exponential family as 
$$
P(\text{RSS}_{1,s}>
\text{PRSS}_{s,\sigma_{a}})
=
P(\text{RSS}_{1,s}/\sigma^2>
\text{PRSS}_{s,\sigma_{a}}/\sigma^2) 
=
P(\chi^2_{s-1}-(s-1)>2(s+2)r^2-(s-1))
=(\star)
$$
To continue, set $f(r,s) = [2r^2\frac{s+2}{s-1}-1]
(s-1)$, that
$$
(\star)
=
P\bigg(\chi^2_{s-1}-(s-1)>[2r^2\frac{s+2}{s-1}-1]
(s-1)\bigg)
\le \exp\bigg(-\frac{f(r,s)}{8}\min\big\{1,2r^2\frac{s+2}{s-1}-1\big\}\bigg)
$$
Given $r>0$, one has for all  $s>0$, 
$$
P\bigg(\chi^2_{s-1}-(s-1)>2(s+2)r^2\bigg)
\le 
\exp\bigg(-\frac{f(r,s)}{8} \bigg).
$$
Now, solve $f(r,s) > 16\log T$ to find $s_{b,3}(T)$. Note $f(r,s) 
= 2r^2(s+2) - (s-1)
= [2r^2 - 1]s
+[4r^2 + 1]$. So choose $s_{a,3}(T) = 16\log T/(2r^2-1)$ to have $s > s_{a,3}(T)$ implies
$P(\text{RSS}_{1,s} > \text{PRSS}_{s,\sigma_{a}})>T^{-2}$ and hence $a_{k,s,3}\le T^{-1}$.
\end{proof}

\subsection{Lemmas on bounding $b_k$.}

\begin{lemma}[Bounding $b_{k,s,1}$ at \eqref{b_ks1}]\label{lm:b_ks1}
For any $s \ge s_{b,1}(T) \equiv 128r^2(\sigma/\Delta_k)^2\log T$,
\begin{equation}
    b_{k,s,1}=0.
\end{equation}
\end{lemma}
\begin{proof}
We first recall  $b_{k,s,1}= E[I(Q_{k,s}(\tau_k)>T^{-1})I(A_{k,s})I(G_{k,s})]$, where 
$Q_{k,s}(x) = P(\hat{\mu}_{k,t}^{*} >x|\mathcal{H}_{k,s})$. Note $\tau_k = \mu_k + \Delta_k/2$. By the definition of the the event $G_{k,s}$,
\begin{eqnarray}
Q_{k,s}(\tau_k)I(G_{k,s})
&=&
P(\hat{\mu}_{k,t}^{*} - \bar{Y}_{k,s}>\tau_k-\bar{Y}_{k,s}|\mathcal{H}_{k,s})
I(G_{k,s})\label{eq:b_ks1_2}\\
&\le& 
\exp(-[2(s+2)^{-2}\text{PRSS}_{s,\sigma_{a}}]^{-1}((\tau_k- \bar{Y}_{k,s})^2/2))I(G_{k,s})\\
&=&
\exp\bigg(-[4r^2(s+2)^{-1}\sigma^2]^{-1}((\tau_k- \bar{Y}_{k,s})^2/2)\bigg)I(G_{k,s}). \label{eq:b_ks1_2}
\end{eqnarray}
We then recast the good event $A_{k,s} $ that the suboptimal arm $k$ is not seriously overestimated as
\begin{eqnarray*}
A_{k,s}
&=&\{\bar{Y}_{k,s}-\mu_k < \Delta_k/4\} \\
&=& 
\{\tau_k-\bar{Y}_{k,s}>\Delta_k/4\}\\
&=&
\{(\tau_k-\bar{Y}_{k,s})^2>\Delta_k^2/16\}\\
&=&
\{
[4r^2(s+2)^{-1}\sigma^2]^{-1}([\tau_k-\bar{Y}_{k,s}]^2/2)
>
128^{-1}r^{-2}(\sigma/\Delta_k)^{-2}(s+2)
\}.
\end{eqnarray*}
Pick $s_{b,1}(T) = 128r^2 (\sigma/\Delta_k)^2 \log T$ to have 
$128^{-1}r^{-2}(\sigma/\Delta_k)^{-2}(s+2)> \log T$
for any
$s > s_{b,1}(T)$. Therefore, equation \eqref{eq:b_ks1_2} on the event $A_{k,s}$ become \begin{equation}
Q_{k,s}(\tau_k)I(A_{k,s})I(G_{k,s})<T^{-1}I(A_{k,s})I(G_{k,s}).
\end{equation}
Thus, we conclude that, for any
$s > s_{b,1}(T)$,
\begin{equation}
b_{k,s,1}=
E[I(Q_{k,s}(\tau_k)>T^{-1})I(A_{k,s})I(G_{k,s})]=0.
\end{equation}
\end{proof}

\begin{lemma}[Bounding $b_{k,s,2}$ at \eqref{b_ks2}]\label{lm:b_ks2}
For any $s \ge s_{b,2}(T) \equiv 32(\sigma/\Delta_k)^2\log T$,
\begin{equation}
    b_{k,s,2}\le T^{-1} .
\end{equation}
\end{lemma}
\begin{proof}
First we note 
$
b_{k,s,2} \le E[I(A_{k,s}^{c})]
= P(\bar{Y}_{k,s}-\mu_{k}>\Delta_k/4).
$
From Gaussian reward, the tail probability of empirical upper deviation is exponentially decaying as
$$P(\bar{Y}_{k,s}-\mu_{k} > \frac{\Delta_k}{4})\le \exp\bigg(-[s^{-1}\sigma^2]^{-1}\frac{(\Delta_k/4)^2}{2}\bigg).$$
Choose $s_{b,2}(T) = 32(\sigma/\Delta_{k})^2\log T$ to have $s > s_{b,2}(T)$ implies $P(\bar{Y}_{k,s}-\mu_k > \Delta_{k}/4)<T^{-1}$.
\end{proof}

\begin{lemma}[Bounding $b_{k,s,3}$ at \eqref{b_ks3}]\label{lm:b_ks3}
For any $s \ge s_{b,3}(T) \equiv 8(2r^2-1)^{-1}\log T$,
\begin{equation}
    b_{k,s,3}\le T^{-1}. 
\end{equation}
\end{lemma}
\begin{proof}
We first note $b_{k,s,3}\le E[I(G_{k,s}^{c})]
=P(\text{RSS}_{k,s}>
\text{PRSS}_{s,\sigma_{a}})$.
From Gaussian reward, the distribution of scaled RSS is chi-square distributed. That is, $\text{RSS}_{k,s}/\sigma^2\sim \chi^2_{s-1}$. 
Also note $\text{PRSS}_{s,\sigma_{a}}/\sigma^2 = 2(s+2)r^2$. From Lemma \ref{cor:chisquare_concentra}, the tail probability of residual sum of square is exponentially decaying as sub-exponential family as 
$$
P(\text{RSS}_{k,s}>
\text{PRSS}_{s,\sigma_{a}})
=
P(\text{RSS}_{k,s}/\sigma^2>
\text{PRSS}_{s,\sigma_{a}}/\sigma^2) 
=
P(\chi^2_{s-1}-(s-1)>2(s+2)r^2-(s-1))
=(\star)
$$
Continue, set $f(r,s) = [2r^2\frac{s+2}{s-1}-1]
(s-1)$, that
$$
(\star)
=
P\bigg(\chi^2_{s-1}-(s-1)>[2r^2\frac{s+2}{s-1}-1]
(s-1)\bigg)
\le \exp\bigg(-\frac{f(r,s)}{8}\min\big\{1,2r^2\frac{s+2}{s-1}-1\big\}\bigg)
$$
Given $r>0$, one has for all  $s>0$, 
$$
P\bigg(\chi^2_{s-1}-(s-1)>2(s+2)r^2\bigg)
\le 
\exp\bigg(-\frac{f(r,s)}{8} \bigg).
$$
Now, solve $f(r,s) > 8\log T$ to find $s_{b,3}(T)$. Note $f(r,s) 
= 2r^2(s+2) - (s-1)
= [2r^2 - 1]s
+[4r^2 + 1]$. So choose $s_{b,3}(T) = 8\log T/(2r^2-1)$ to have $s > s_{b,3}(T)$ implies
$P(\text{RSS}_{k,s} > \text{PRSS}_{s,\sigma_{a}})<T^{-1}$.
\end{proof}


\section{Supporting Lemmas}
\label{Appe:supp_lemmas}
\begin{lemma}[Lower bound of Gaussian tail]\label{lm:Gaussian_lower_bound}
Set $Z \sim N(0,1)$. Then,
\begin{equation}
    P(Z \ge t) \ge
    \begin{cases}
    \exp(-\frac{3}{2}t^2) &\text{on} \{t \ge \sqrt{2\pi}\}\\
    1-\Phi(\sqrt{2\pi})
    &\text{on} \{0 < t < \sqrt{2\pi}\}
    \end{cases}.
\end{equation}
\end{lemma}
\begin{proof}
From vanilla Gaussian lower bound
\begin{equation*}
P(Z \ge t) \ge 
\begin{cases}
    \frac{1}{\sqrt{2\pi}}
    \frac{t}{1+t^2}\exp(-\frac{1}{2}t^2) &\text{on} \{t \ge c\}\\
    1-\Phi(c)
    &\text{on} \{0 < t < c\}
    \end{cases},
\end{equation*}
Inequality $1+t^2 \le \exp(t^2)$ for all $t \ge 0$ implies $t/(1+t^2) \ge c \exp(-t^2)$ on $\{t \ge c\}$. Pick $c = \sqrt{2\pi}$ to have the claim.
\end{proof}

\begin{lemma}[Upper bound of reciprocal of Gaussian tail]\label{cor:Gau_lower}
Set $Z \sim N(0,1)$. Then, 
\begin{equation}
    P(Z \ge t)^{-1} \le
    \begin{cases}
    \exp(\frac{3}{2}t^2) &\text{on} \{t \ge \sqrt{2\pi}\}\\
    [1-\Phi(\sqrt{2\pi})]^{-1}
    &\text{on} \{0 < t < \sqrt{2\pi}\}
    \end{cases}.
\end{equation}
\end{lemma}

\begin{lemma}[Proposition 2.9 in \cite{wainwright2019high}]\label{lemma:subExp}
Suppose that $X$ is sub-exponential with parameters $(\nu, \alpha)$. Then
$$P(X-\mu \ge t) \le \exp(-\frac{t}{2\alpha}\min\{1, \frac{\alpha t}{\nu^2}\}).$$
\end{lemma}

\begin{lemma}[Concentration of Chi-Square distribution]\label{cor:chisquare_concentra}
Set $Y \sim \chi^2_{n}$. Then, 

\begin{equation}
    P(\chi^2_{n}-n \ge t)
    \le 
    \exp(-\frac{t}{8}\min\{1, \frac{t}{n}\})
\end{equation}
\end{lemma}

\begin{proof}
Note that $\chi^2_{n}$ is sub-exponential with parameter $(\nu, \alpha) = (2\sqrt{n}, 4)$. Take $\nu^2/\alpha = n$ in Lemma \ref{lemma:subExp} to get
\begin{equation}
    P(\chi^2_n - n \ge t ) \le
    \begin{cases}
    \exp(-\frac{t^2}{8n})&\text{ if } 0 \le t \le n,\\
    \exp(-\frac{t}{8})&\text{ if }  t> n,
    \end{cases}
\end{equation}
and hence 
$P(\chi^2_{n}-n \ge t)
    \le 
    \exp(-\frac{t}{8}\min\{1, \frac{t}{n}\})$.
\end{proof}

\end{document}